%% file: Manuscript_HMC.tex
\documentclass{article}

\usepackage[final,nonatbib]{neurips_2021}

\usepackage{microtype,algorithmic,algorithm}
\usepackage{graphicx}
\usepackage{booktabs} 
\usepackage{url}
\usepackage[utf8]{inputenc} 
\usepackage[T1]{fontenc}    
\usepackage{amsfonts,dsfont} 
\usepackage{amssymb,amsmath,amsthm}
\usepackage{nicefrac,enumitem}
\usepackage{euscript}
\usepackage{subfig}
\usepackage{microtype}
\usepackage{caption,titlesec}
\usepackage{lipsum}
\usepackage{xcolor}

\input{math_commands.tex}

\newtheorem{definition}{Definition}
\newtheorem{lemma}{Lemma}
\newtheorem{theorem}{Theorem}

\newtheorem{proofoflemma}{Proof of Lemma}
\newtheorem{proofoftheorem}{Proof of Theorem}
\graphicspath{{./Images/}}
\title{
On Using Hamiltonian Monte Carlo Sampling for Reinforcement Learning Problems in High-dimension
}

\author{%
  Udari Madhushani$^{\dagger}$, Biswadip Dey$^{\S}$, Naomi Ehrich Leonard$^{\dagger}$, Amit Chakraborty$^{\S}$
  \\
  $^{\dagger}$ Princeton University | $^{\S}$ Siemens Corporation, Technology
  \\
  \texttt{udarim@princeton.edu}, \texttt{biswadip.dey@siemens.com}
  \\
  \texttt{naomi@princeton.edu}, \texttt{amit.chakraborty@siemens.com}
}

\begin{document}

\maketitle

\begin{abstract}
Value function based reinforcement learning (RL) algorithms, for example, $Q$-learning, learn optimal policies from datasets of actions, rewards, and state transitions. However, when the underlying state transition dynamics are stochastic and evolve on a high-dimensional space, generating independent and identically distributed (IID) data samples for creating these datasets poses a significant challenge due to the intractability of the associated normalizing integral. In these scenarios, Hamiltonian Monte Carlo (HMC) sampling offers a computationally tractable way to generate data for training RL algorithms. In this paper, we introduce a framework, called \textit{Hamiltonian $Q$-Learning}, that demonstrates, both theoretically and empirically, that $Q$ values can be learned from a dataset generated by HMC samples of actions, rewards, and state transitions. Furthermore, to exploit the underlying low-rank structure of the $Q$ function, Hamiltonian $Q$-Learning uses a matrix completion algorithm for reconstructing the updated $Q$ function from $Q$ value updates over a much smaller subset of state-action pairs. Thus, by providing an efficient way to apply $Q$-learning in stochastic, high-dimensional settings, the proposed approach broadens the scope of RL algorithms for real-world applications.
\end{abstract}
%
%
%

%
%
\section{Introduction}
In recent years, reinforcement learning has shown remarkable success with sequential decision-making tasks wherein an agent, after observing the current state of the environment, chooses an action to receive a reward, and subsequently, the environment transitions to a new state \cite{mnih2015human,sutton2018reinforcement}. RL has been applied to a variety of problems, such as automatic control \cite{duan2016benchmarking}, robotics \cite{kober2013reinforcement}, resource allocation \cite{mao2016resource}, and chemical process optimization \cite{zhou2017optimizing}. However, existing model-free RL approaches typically perform well only when the environment has been explored long enough, and the algorithm has used a large number of samples in the process \cite{kamthe2018data,yang2020data}. $Q$-learning is a model-free RL approach where an agent chooses its actions based on a policy defined by the state-action value function, i.e., the $Q$ function \cite{watkins1989learning,watkins1992q}. The performance of $Q$-learning algorithms depends strongly on the ability to access data samples, which can provide accurate estimates of the expected $Q$ values.

As these algorithms compute the expected $Q$ values by calculating the sample mean of $Q$ values over a set of IID samples, they assume access to a simulator that can generate IID samples according to the state transition probability. However, when the state transition probability distribution is high-dimensional, generating IID samples poses a significant challenge due to - (i) lack of closed-form solutions, and (ii) insufficiency of deterministic approximations, of the normalizing integral, preventing the utilization of existing RL methods. This motivated us to ask - \textit{How can we develop value function based RL methods when generating IID samples is impractical?} 

A crucial step in developing such methods is identifying means to draw samples from an unnormalized distribution. Importance sampling methods offer techniques to draw samples from a distribution without computing the corresponding normalizing integral. Hamilton Monte Carlo (HMC) sampling is one such method; it allows one to generate samples from the unnormalized state transition distribution \cite{neal2011mcmc}. Equipped with HMC, we attempt to answer the following question: \textit{How can we combine HMC sampling with $Q$-Learning to learn optimal policies for high-dimensional problems?}

In this work, we introduce \textit{Hamiltonian $Q$-Learning} to answer this question. We show that Hamiltonian $Q$-Learning can infer optimal policies even when it calculates the expected $Q$ values using HMC samples instead of IID samples. Now, even though HMC samples overcome the challenges associated with drawing IID samples in high-dimensions, a large number of samples is still needed to learn the $Q$ function because high-dimensional spaces often lead to a large number of state-action pairs. We address this issue by leveraging matrix completion techniques. It has been observed that formulating planning and control tasks in a variety of problems, such as video games (e.g., Atari games) and classical control problems (e.g., simple pendulum, cart pole) as $Q$-Learning problems leads to low-rank structures in the $Q$ matrix associated with the problem \cite{ong2015value,Yang2020Harnessing,shah2020sample}. Since these systems naturally consist of a large number of states, exploiting the low-rank structure in the $Q$ matrix in an informed way can enable further reduction in the computational complexity. \textit{Hamiltonian $Q$-Learning} uses matrix completion to reconstruct the $Q$ matrix from a small subset of expected $Q$ values making it data-efficient.

The three main contributions of this work are threefold. 
\textit{First}, we introduce a modified $Q$-learning framework, called \textit{Hamiltonian $Q$-learning}, which uses HMC sampling for efficient computation of the $Q$ values. This innovation, by proposing to sample $Q$ values from the region with the dominant contribution to the expectation of discounted reward, provides a data-efficient approach for using $Q$-learning in real-world problems with high-dimensional state space and probabilistic state transition. Integration of this sampling approach with matrix-completion enables us to update $Q$ values for only a small subset of state-action pairs and reconstruct the complete $Q$ matrix.
\textit{Second}, we provide theoretical guarantees that the error between the optimal $Q$ function and the $Q$ function computed by updating $Q$ values using HMC sampling can be made arbitrarily small. This result holds even when only a small fraction of the $Q$ values are updated using HMC samples and the rest are estimated using matrix completion. We also provide theoretical guarantee that the sampling complexity of our algorithm matches the mini-max sampling complexity proposed by \cite{tsybakov2008introduction}.
\textit{Finally}, we apply Hamiltonian $Q$-learning to a high-dimensional problem (in particular, the problem of stabilizing a double pendulum on a cart) as well as to benchmark control tasks (inverted pendulum, double integrator, cartpole, and acrobot). Our results show that the proposed approach becomes more effective with increase in state space dimension.

\paragraph{Related Work:}
The last decade has witnessed a growing interest in improving sample efficiency in RL methods by exploiting emergent global structures from underlying system dynamics. \cite{kamthe2018data, deisenroth2011pilco, pan2014probabilistic, buckman2018sample} have proposed model-based RL methods that improve sample efficiency by explicitly incorporating prior knowledge about state transition dynamics of the underlying system. \cite{dearden1998bayesian, koppel2018nonparametric, jeong2017assumed} propose Baysean methods to approximate the $Q$ function. \cite{ong2015value,Yang2020Harnessing} consider a model-free RL approach that exploit structures of state-action value function. The work by \cite{ong2015value} decomposes the $Q$ matrix into a low-rank and sparse matrix model and uses matrix completion methods \cite{candes2010matrix, wen2012solving, chen2018harnessing} to improve sample efficiency. A more recent work \cite{Yang2020Harnessing} has shown that incorporating low rank matrix completion methods to recover $Q$ matrix from a small subset of $Q$ values can improve learning of optimal policies. At each time step the agent chooses a subset of state-action pairs and update the corresponding $Q$ value using the Bellman optimally equation that considers a discounted average between reward and expectation of the $Q$ values of next states. \cite{shah2020sample} extends this work by proposing a novel matrix estimation method and providing theoretical guarantees for the convergence to a $\epsilon$-optimal $Q$ function. On the other hand, entropy regularization techniques penalize excessive randomness in the conditional distribution of actions for a given state and provide an alternative means to implicitly exploit the underlying low-dimensional structure of the value function \cite{ahmed2019understanding,yang2019regularized,smirnova2020convergence}. \cite{lee2019sample} has proposed an approach that samples a whole episode and then updates values in a recursive, backward manner.

%
%
\section{Preliminary Concepts} 
In this section, we provide a brief background on $Q$-Learning, HMC sampling and matrix completion, as well as introduce the mathematical notations. In this paper, $|\EuScript{Z}|$ denotes the cardinality of a set $\EuScript{Z}$. Moreover, $\mathds{R}$ represent the real line and $A^T$ denotes the transpose of matrix $A$.
\subsection{$Q$-Learning}
Markov Decision Process (MDP) is a mathematical formulation that captures salient features of sequential decision making \cite{bertsekas1995dynamic}. In particular, a \textit{finite MDP} is defined by the tuple $(\mathcal{S},\mathcal{A},\mathbb{P},r,\gamma)$, where $\mathcal{S}$ is the finite set of system states, $\mathcal{A}$ is the finite set of actions, $\mathbb{P}: \mathcal{S}\times \mathcal{A}\times \mathcal{S} \to [0,1]$ is the transition probability kernel, $r:\mathcal{S}\times \mathcal{A}\to \mathds{R}$ is a bounded reward function, and $\gamma \in [0,1)$ is a discounting factor. Without loss of generality, states $s\in \mathcal{S}$ and actions $a\in\mathcal{A}$ can be assumed to be $\EuScript{D}_s$-dimensional and $\EuScript{D}_a$-dimensional real vectors, respectively. Moreover, by letting $s^i$ denote the $i$th element of a state vector, we define the range of state space in terms of the following intervals $[d_i^-,d_i^+]$ such that $s^i\in [d_i^-,d_i^+]$ $\forall i\in \{1,\ldots,\EuScript{D}_s\}$. At each time $t\in \{1,\ldots, T\}$ over the decision making horizon, an agent observes the state of the environment $s_t\in\mathcal{S}$ and takes an action $a_t$ according to some policy $\pi$ which maximizes the discounted cumulative reward. Once this action has been executed, the agent receives a reward $r(s_t,a_t)$ from the environment and the state of the environment changes to $s_{t+1}$ according to the transition probability kernel $\mathbb{P}\left(\cdot |s_t,a_t\right)$. The $Q$ function, which represents the expected discounted reward for taking a specific action at the current time and following the policy thereafter, is defined as a mapping from the space of state-action pairs to the real line, i.e. $Q:\mathcal{S}\times \mathcal{A}\to \mathds{R}$. Then, by letting $Q^t$ represent the $Q$ matrix at time $t$, i.e. the tabulation of $Q$ function over all possible state-action pairs associated with the finite MDP, we can express the $Q$ value iteration over time steps as
\begin{align}
& Q^{t+1}(s_t,a_t)
=
\sum_{s\in \mathcal{S}}\mathbb{P}\left(s | s_t,a_t\right)\left(r(s_t,a_t)+\gamma\max _{a}Q^t(s,a)\right).
\label{eq:bellman}
\end{align}
Under this update rule, the $Q$ function converges to its optimal value $Q^*$ \cite{melo2001convergence}. To compute this sum (\ref{eq:bellman}) over possible next states, existing methods rely on either exhaustive sampling or a simulator generating IID samples. However they fail in high-dimensional spaces due to prohibitively high computational cost associated with calculating the normalizing integral of state transition distribution.
\subsection{Hamiltonian Monte Carlo}
\label{subsec:HMC}
Hamiltonian Monte Carlo is an efficient sampling approach for drawing samples from probability distributions known up to a constant, i.e., unnormalized distributions. It offers faster convergence than Markov Chain Monte Carlo (MCMC) sampling \cite{neal2011mcmc, betancourt2017conceptual, betancourt2017geometric,neklyudov2020involutive}. To draw samples from a smooth target distribution $\mathcal{P}(s)$, which is defined on the Euclidean space and assumed to be known up to a constant, HMC extends the target distribution to a joint distribution over the target variable $s$ (viewed as position within the HMC context) and an auxiliary variable $v$ (viewed as momentum within the HMC context). We define the Hamiltonian of the system as
$
H(s,v)
= - \log\mathcal{P}(s,v)
= - \log\mathcal{P}(s) - \log\mathcal{P}(v|s)
= U(s) + K(v,s),
$
where $U(s) \triangleq -\log \mathcal{P}(s)$ and $K(v,s) \triangleq -\log\mathcal{P}(v|s) = \frac{1}{2}v^TM^{-1}v$ represent the potential and kinetic energy, respectively, and $M$ is a suitable choice of the mass matrix.

HMC sampling method consists of the following \textit{three} steps $-$ (i) a new momentum variable  $v$ is drawn from a fixed probability distribution, typically a multivariate Gaussian; (ii) then a new proposal $(s^{\prime},v^{\prime})$ is obtained by generating a trajectory that starts from $(s,v)$ and obeys Hamiltonian dynamics, i.e. $\dot{s} = \frac{\partial H}{\partial v}, \dot{v} = -\frac{\partial H}{\partial s}$; and (iii) finally this new proposal is accepted with probability $\min \left\{1, \exp\left(H(s,v)-H(s^{\prime},-v^{\prime})\right)\right\}$ following the Metropolis–Hastings acceptance/rejection rule.

Thus HMC sampling offers a way to draw samples from unnormalized transition distributions often encountered in high-dimensional state spaces. However, since such problems often consist of a large number of state-action pairs, learning the $Q$ function still requires a large number of samples. This leads to poor sample efficiency.
\subsection{Low-rank Structure in $Q$-learning and Matrix Completion}
When a matrix is low-rank or has a sparse structure, matrix completion methods can reconstruct it accurately from a small subset of entries. Prior work \cite{johns2007constructing, geist2013algorithmic,ong2015value,shah2020sample} on value function approximation based approaches for RL has implicitly assumed that the state-action value functions are low-dimensional and used various basis functions to represent them, e.g. CMAC, radial basis function, etc. This can be attributed to the fact that the underlying state transition and reward function are often endowed with some structure. More recently, \cite{Yang2020Harnessing} provide empirical guarantees that the $Q$-matrices for benchmark Atari games and classical control tasks exhibit low-rank structure. 

Therefore, using matrix completion techniques \cite{xu2013parallel,chen2018harnessing} to recover $Q \in \mathds{R}^{|\mathcal{S}|\times|\mathcal{A}|}$ from few observed $Q$ values constitutes a viable approach towards improving sample efficiency. As low-rank matrix structures can be recovered by constraining the nuclear norm (i.e., the sum of its singular values), the $Q$ matrix can be reconstructed from its observed values ($\hat{Q}$) by solving
\begin{equation}
\begin{aligned}
{Q} = \argmin_{\widetilde{Q}\in \mathds{R}^{|\mathcal{S}|\times|\mathcal{A}|}} \quad & \|\widetilde{Q}\|_*
\\
\textrm{subject to} \quad & \EuScript{J}_{\Omega}(\widetilde{Q})=\EuScript{J}_{\Omega}(\hat{Q})
\end{aligned}
\label{eq:matrxComp}
\end{equation}
where $\|\cdot\|_*$ denotes the nuclear norm, $\Omega$ is the observed set of elements, and $\EuScript{J}_{\Omega}$ is the observation operator, i.e. $\EuScript{J}_{\Omega}(x)=x$ if $x\in \Omega$ and zero otherwise.
%
%
\section{Hamiltonian $Q$-Learning}
A large class of real world sequential decision making problems - for example, board/video games, control of a robot's movement, and portfolio optimization - involves high-dimensional state spaces and often has large number of distinct states along each individual dimension. As using a $Q$-Learning based approach to train RL-agents for these problems typically requires tens to hundreds of millions of samples \cite{mnih2015human, silver2017mastering}, there is a strong need for sample efficient algorithms for $Q$-Learning. In addition, state transition in such systems is often probabilistic in nature; even when the underlying dynamics of the system is inherently deterministic; presence of external disturbances and parameter variations/uncertainties lead to probabilistic state transitions.

Learning an optimal $Q^*$ function through value iteration methods requires updating $Q$ values of state-action pairs using a sum of the reward and a discounted expectation of $Q$ values associated with next states. In this work, we assume the reward to be a deterministic function of state-action pairs. However, when the reward is stochastic, these results can be extended by replacing the reward with its expectation. Subsequently, we can express (\ref{eq:bellman}) as
\begin{align}
Q^{t+1}(s_t,a_t)=r(s_t,a_t)+\gamma\mathbb{E}\left(\max _{a}Q^t(s,a)\right),
\label{eq:bellUpdate} 
\end{align}
where $\mathbb{E}$ denotes the expectation over the discrete probability measure $\mathbb{P}$. When the underlying state space is high-dimensional and has large number of states, we encounter two key challenges while attempting to learn the $Q$ function: (\textit{i}) difficulty in estimating the expectation in (\ref{eq:bellUpdate}) due to high computational cost of exhaustive sampling and impracticality of generating IID samples; and (\textit{ii}) a sample complexity that increases quadratically with the number of states and linearly with the number of actions.
%

To the best of our knowledge, \textit{Hamiltonian $Q$-Learning} offers the first solution to this problem by combining \textit{HMC sampling} and \textit{matrix completion} that overcome the first and the second challenge, respectively.
%
%
%
%
%

\subsection{HMC sampling for learning $Q$ function}
%
%
A number of importance-sampling methods \cite{liu1996metropolized, betancourt2017conceptual} have been developed for estimating the expectation of a function by drawing samples from the region with the dominant contribution to the expectation. HMC is one such importance-sampling method that draws samples from the typical set, i.e., the region that maximizes probability mass, which provides the dominated contribution to the expectation. Since the decay in $Q$ function is significantly smaller compared to the typical exponential or power law decays in transition probability function, HMC provides a better approximation for the expectation of the $Q$ value of the next states \cite{Yang2020Harnessing,shah2020sample}. Then by letting $\mathcal{H}_t$ denote the set of HMC samples drawn at time step $t$, we update the $Q$ values as:
\begin{align}
Q^{t+1}(s_t,a_t)=r(s_t,a_t)+\frac{\gamma}{|\mathcal{H}_t|}\sum_{s\in \mathcal{H}_t}\max _{a}Q^t(s,a).
\label{eq:HMCUpdate}
\end{align}
%
%

\paragraph{HMC for a smooth truncated target distribution:}
%
Recall that region of states is a subset of a Euclidean space given as $s\in [d_1^-,d_1^+]\times \ldots \times [d_{\EuScript{D}_s}^-,d_{\EuScript{D}_s}^+]\subset \mathds{R}^{\EuScript{D}_s}$. Thus the main challenge to using HMC sampling is to define a smooth continuous target distribution $\mathcal{P}(s|s_t,a_t)$ which is defined on $\mathds{R}^{\EuScript{D}_s}$ with a sharp decay at the boundary of the region of states \cite{yi2017roll, chevallier2018hamiltonian}. In this work, we generate the target distribution by first defining the transition probability kernel from the conditional probability distribution defined on $\mathds{R}^{\EuScript{D}_s}$ and then multiplying it with a smooth cut-off function.

We first consider a probability distribution $\EuScript{P}(\cdot |s_t,a_t):\mathds{R}^{\EuScript{D}_s}\to \mathds{R}$ such that the following holds
\begin{align}
\mathbb{P}(s|s_t,a_t) \propto \int_{s-\varepsilon}^{s+\varepsilon} \EuScript{P}(s |s_t,a_t)ds \label{eq:probDis}
\end{align}
for some arbitrarily small $\varepsilon>0$. Then the target distribution can be defined as
\begin{align}
\mathcal{P}(s|s_t,a_t)=\EuScript{P}(s |s_t,a_t)&\prod_{i=1}^{\EuScript{D}_s}\left[\frac{1}{1+\exp(-\kappa (d^+_i-s^i))} \cdot \frac{1}{1+\exp(-\kappa (s^i-d^-_i))}\right ].\label{eq:targetDis}
\end{align}
Note that there exists a large $\kappa>0$ such that if $s\in [d_1^-,d_1^+]\times \ldots \times [d_{\EuScript{D}_s}^-,d_{\EuScript{D}_s}^+]$ then $\mathcal{P}(s|s_t,a_t)\propto \mathbb{P}(s|s_t,a_t)$ and $\mathcal{P}(s|s_t,a_t)\approx 0$ otherwise. Let $\mu(s_t,a_t), \Sigma(s_t,a_t)$ be the mean and covariance  of the transition probability kernel. In this paper we consider transition probability kernels of the form
\begin{align}
&\mathbb{P}(s|s_t,a_t)
\propto
\exp\left(-\frac{1}{2}(s-\mu(s_t,a_t))^T\Sigma^{-1}(s_t,a_t)(s-\mu(s_t,a_t))\right).
\label{eq:stateTrans}
\end{align}
Then from (\ref{eq:probDis}) the corresponding mapping can be given as a multivariate Gaussian $\EuScript{P}(s |s_t,a_t)=\mathcal{N}(\mathbf{\mu}\left(s_t,a_t),{\Sigma}(s_t,a_t)\right).$ Thus from (\ref{eq:targetDis}) it follows that the target distribution is
\begin{align}
&
\mathcal{P}(s|s_t,a_t)=\mathcal{N}(\mathbf{\mu}\left(s_t,a_t),{\Sigma}(s_t,a_t)\right)
\prod_{i=1}^{\EuScript{D}_s}\frac{1}{1+\exp(-\kappa (d^+_i-s^i))}\frac{1}{1+\exp(-\kappa (s^i-d^-_i))}.
\label{eq:targetDisNormal}
\end{align}
\paragraph{Choice of potential energy, kinetic energy and mass matrix:}
For brevity of notation we drop the explicit dependence of $\mathcal{P}(\cdot)$ on $(s_t,a_t)$ and denote the target distribution as $\mathcal{P}(s)$ defined over the Euclidean space $\mathds{R}^{\mathcal{D}_s}$. As explained in Section \ref{subsec:HMC} we choose the potential energy as
\begin{align*}
U(s) 
= 
-\log(\mathcal{P}(s)) 
=& 
\frac{1}{2}(s-\mu)^T\Sigma^{-1}(s-\mu) - \frac{1}{2}\log\Big((2\pi)^{D_s}\det(\Sigma)\Big) 
\\
&\quad 
- \sum_{i=1}^{D_s}\left[\log\Big(1+\exp(-\kappa (d^+_i-s^i))\Big) + \log\Big(1+\exp(-\kappa (s^i-d^-_i))\Big)\right].
\end{align*}
We consider an Euclidean metric  $\EuScript{M}$ that induces the distance between $\tilde{s},\bar{s}$ as $d(\tilde{s},\bar{s}) = (\tilde{s}-\bar{s})^T \EuScript{M} (\tilde{s}-\bar{s})$. Then we define $\EuScript{M}_s \in \mathds{R}^{\mathcal{D}_s\times \mathcal{D}_s}$ as a diagonal scaling matrix and $\EuScript{M}_r \in \mathds{R}^{\mathcal{D}_s\times \mathcal{D}_s}$ as a rotation matrix in dimension $\mathcal{D}_s$. With this we can define $M$ as $M = \EuScript{M}_r \EuScript{M}_s \EuScript{M} \EuScript{M}_s^T \EuScript{M}_r^T$. Thus, any metric $M$ that defines an Euclidean structure on the target variable space induces an inverse structure $d(\tilde{v},\bar{v}) = (\tilde{v}-\bar{v})^TM^{-1}(\tilde{v}-\bar{v})$ on the momentum variable space. This generates a natural family of multivariate Guassian distributions such that $\mathcal{P}(v|s)= \mathcal{N}(0,M)$ leading to the kinetic energy $K(v,s) = -\log \mathcal{P}(v|s) = \frac{1}{2}v^TM^{-1}v$ where $M^{-1}$ is the covariance of the target distribution. 
%
%
\subsection{$Q$-Learning with HMC and matrix completion}
In this work we consider problems with a high-dimensional state space and large number of distinct states along individual dimensions. Although these problems admit a large $Q$ matrix, we can exploit low rank structure of the $Q$ matrix to further improve the sample efficiency.

At each time step $t$ we randomly sample a subset $\Omega_t$ of state-action pairs (each state-action pair is sampled independently with some probability $p$) and update the $Q$ function for state-action pairs in $\Omega_t$. Let $\widehat{Q}^{t+1}$ be the updated $Q$ matrix at time $t$. Then from (\ref{eq:HMCUpdate}) we have
\begin{align}
\widehat{Q}^{t+1}(s_t,a_t)=r(s_t,a_t)+\frac{\gamma}{|\mathcal{H}_t|}\sum_{s\in \mathcal{H}_t}\max _{a}Q^t(s,a), 
\label{eq:HMCSub}
\end{align}
for any $(s_t,a_t) \in \Omega_t$. Then we recover the complete matrix $Q^{t+1}$ by using the  method given in (\ref{eq:matrxComp}). Thus we have
\begin{equation}
\begin{aligned}
Q^{t+1} 
= \argmin_{\widetilde{Q}^{t+1}\in \mathds{R}^{|\mathcal{S}|\times |\mathcal{A|}}} 
\; & 
\|\widetilde{Q}^{t+1}\|_*
\\
\textrm{subject to} 
\; & 
\EuScript{J}_{\Omega_t}\left(\widetilde{Q}^{t+1}\right)=\EuScript{J}_{\Omega_t}\left(\widehat{Q}^{t+1}\right)
\end{aligned}
\label{eq:HMCMC}
\end{equation}
Similar to the approach used by \cite{Yang2020Harnessing}, we approximate the rank of the $Q$ matrix as the minimum number of singular values that are needed to capture 99\% of its nuclear norm.
%
\setlength{\textfloatsep}{2pt}
\begin{algorithm}[t!]
    \caption{Hamiltonian $Q$-Learning}
    \begin{algorithmic}\label{alg:HMCQAlg}
      \STATE {\bfseries Inputs:} Discount factor $\gamma$; Range of state space; Time horizon $T$;
      \\
      \STATE {\bfseries Initialization:} Randomly initialize $Q^{0}$\\
      \FOR{$t=1$ {\bfseries to} $T$}
      \STATE \textbf{Step 1}: Randomly sample a subset of state-action pairs $\Omega_t$ \\
      \STATE \textbf{Step 2}: \textbf{HMC sampling phase} - Sample a set of next states $\mathcal{H}_t$ according to the target distribution defined in (\ref{eq:targetDis}) \\
      \STATE \textbf{Step 3}: \textbf{Update phase} -  For all $(s_t,a_t)\in \Omega_t$\\ 
    	$\displaystyle \widehat{Q}^{t+1}(s_t,a_t)=r(s_t,a_t)+\frac{\gamma}{|\mathcal{H}_t|}\sum_{s\in \mathcal{H}_t}\max_a Q^t(s,a)$
    	\STATE \textbf{Step 4}: {\bf{Matrix Completion phase }} \\
    	$\displaystyle \begin{aligned} Q^{t+1} = \argmin_{\widetilde{Q}^{t+1}\in \mathds{R}^{|\mathcal{S}|\times |\mathcal{A|}}} \; & \|\widetilde{Q}^{t+1}\|_* \\ \textrm{subject to} \; & \EuScript{J}_{\Omega_t}\left(\widetilde{Q}^{t+1}\right)=\EuScript{J}_{\Omega_t}\left(\widehat{Q}^{t+1}\right) \end{aligned}$
      \ENDFOR
    \end{algorithmic}
\end{algorithm}
\vspace{1em}
\subsection{Convergence, Boundedness and Sampling Complexity}
In this section we provide the main theoretical results of this paper. First, we formally introduce the
following \textit{regularity assumptions}:\\
\indent(\textbf{A1}) 
The state space $\mathcal{S} \subseteq \mathds{R}^{\mathcal{D}_s}$ and the action space $\mathcal{A} \subseteq \mathds{R}^{\mathcal{D}_a}$ are compact subsets.
\\
\indent(\textbf{A2}) 
The reward function is bounded, i.e., $r(s,a)\in [R_{\min},R_{\max}]$ for all $(s,a)\in \mathcal{S}\times \mathcal{A}$.
\\
\indent(\textbf{A3}) 
The optimal value function $Q^*$ is $C$-Lipschitz, i.e.
\begin{align*}
\Big | Q^*(s,a)-Q^*(s^{\prime},a^{\prime})\Big |
\leq 
C \Big( ||s-s^{\prime}||_{F}+||a-a^{\prime}||_{F}\Big)
\end{align*}
where $||\cdot||_F$ is the Frobenius norm (which is same as the Euclidean norm for vectors).

We provide theoretical guarantees that Hamiltonian $Q$-Learning converges to an $\epsilon$-optimal $Q$ function with $\widetilde{O}\left(\frac{1}{\epsilon^{\mathcal{D}_s+\mathcal{D}_a+2}}\right)$ number of samples. This matches the mini-max lower bound $\Omega\left(\frac{1}{\epsilon^{\mathcal{D}_s+\mathcal{D}_a+2}}\right)$ proposed in \cite{tsybakov2008introduction}. First we define a family of $\epsilon$-optimal $Q$ functions as follows.

\begin{definition}[\textbf{$\epsilon$-optimal $Q$ functions}]
Let $Q^*$ be the unique fixed point of the Bellman optimality equation given as $(\mathcal{T}Q)(s^{\prime},a^{\prime})=\sum_{s\in \mathcal{S}}\mathbb{P}(s|s^{\prime},a^{\prime})\left(r(s^{\prime},a^{\prime})+\gamma\max_a Q(s,a) \right) \:\:\forall (s^{\prime},a^{\prime}) \in \mathcal{S}\times \mathcal{A}$ where $\mathcal{T}$ denotes the Bellman operator. Then, under update rule (\ref{eq:bellUpdate}), the $Q$ function almost surely converges to the optimal $Q^*$. We define $\epsilon$-optimal $Q$ functions as the family of functions $\mathbf{Q_{\epsilon}}$ such that $\|Q^{\prime}-Q^*\|_{\infty}\leq \epsilon$ whenever $Q^{\prime}\in \mathbf{Q_{\epsilon}}$.
\end{definition}
As $\|Q^{\prime}-Q^*\|_{\infty}=\max_{(s,a)\in\mathcal{S}\times \mathcal{A}}\|Q^{\prime}(s,a)-Q^*(s,a)\|$, any $\epsilon$-optimal $Q$ function is element wise $\epsilon$-optimal. Our next result shows that under HMC sampling rule given in Step 3 of the Hamiltonian $Q$-Learning algorithm (Algorithm \ref{alg:HMCQAlg}), the $Q$ function converges to the family of $\epsilon$-optimal $Q$ functions.
\begin{theorem}[\textbf{Convergence of $Q$ function under HMC}]
\label{thm:converge}
Let $\EuScript{T}$ be an optimality operator under HMC given as $(\EuScript{T}Q)(s^{\prime},a^{\prime}) = r(s^{\prime},a^{\prime}) + \frac{\gamma}{|\mathcal{H}|}\sum_{s\in \mathcal{H}}\max _{a}Q(s,a), \:\:\forall (s^{\prime},a^{\prime})\in \mathcal{S}\times \mathcal{A},$ where $\mathcal{H}$ is a subset of next states sampled using HMC from the target distribution given in (\ref{eq:targetDis}). Then, under update rule (\ref{eq:HMCUpdate}) and for any given $\epsilon \geq 0$, there exists $n_{\mathcal{H}},t^{\prime}>0$ such that $\|Q^{t}-Q^*\|_{\infty}\leq \epsilon$ $\forall t\geq t^{\prime}$.
\end{theorem}

\begin{proof} (\textit{sketch})
We follow a similar approach to $Q$-function convergence proof, i.e. convergence under exhaustive sampling, with a key modification that accounts for the error incurred by HMC sampling. We notice that $Q$-function error under HMC sampling can be upper bounded by the summation of (\textit{i}) $Q$-function error under exhaustive sampling and (\textit{ii}) the error between empirical average under HMC sampling and expectation under exhaustive sampling. We note that when $Q$-function is Lipschitz from central limit theorem for HMC sampling we can upper bound the cumulative error induced by the second term using a constant. Please refer the Supplementary Material for a detailed proof of this theorem.
\end{proof}

 The next theorem shows that the $Q$ matrix estimated via a suitable matrix completion technique lies in the $\epsilon$-neighborhood of the corresponding $Q$ function obtained via exhaustive sampling.

\begin{theorem}[\textbf{Bounded Error under HMC with Matrix Completion}]
\label{thm:convergeMat}
Let $Q^{t+1}_{\mathcal{E}}(s_t,a_t) = r(s_t,a_t) + \gamma\sum_{s\in \mathcal{S}}\mathbb{P}(s|s_t,a_t)\max _{a}Q_{\mathcal{E}}^t(s,a), \forall (s_t,a_t)\in \mathcal{S}\times \mathcal{A}$ be the update rule under exhaustive sampling, and $Q^t$ be the $Q$ function updated according to Hamiltonian $Q$-Learning (\ref{eq:HMCSub})-(\ref{eq:HMCMC}). Then, for any given $\tilde{\epsilon} \geq 0$, there exists $n_{\mathcal{H}}=\min_{\tau} |\mathcal{H}_{\tau}|,t^{\prime}>0$, such that $\|Q^{t}-Q_{\mathcal{E}}^t\|_{\infty} \leq \tilde{\epsilon}$ $\forall t\geq t^{\prime}$.
\end{theorem}

\begin{proof} (\textit{sketch})
Due to boundedness under matrix completion we notice that error between $Q$ functions updated according to Hamiltonian $Q$-Learning and exhaustive sampling can be upper bounded using summation of  (\textit{i}) error between updated $\widehat{Q}^t$ and optimal function $Q^*$ and (\textit{ii}) error between updated function $Q_{\mathcal{E}}^t$ under exhaustive sampling and optimal function $Q^*$. Proof follows from upper bounding first term using matrix completion boundedness results and second term using Theorem \ref{thm:converge}. Please refer Supplementary Material
for a detailed proof of this theorem. 
\end{proof}

Finally we provide guarantees on the sampling complexity of Hamiltonian $Q$-Learning algorithm.
\begin{theorem}{\bf{(Sampling complexity of Hamiltonian $Q$-Learning)}}
\label{thm:sampleMat}
Let $\mathcal{D}_s$, $\mathcal{D}_a$ be the dimension of state space and action space, respectively. Consider the Hamiltonian $Q$-Learning algorithm presented in Algorithm \ref{alg:HMCQAlg}. Then, under a suitable matrix completion method, the $Q$ function converges to the family of $\epsilon$-optimal $Q$ functions with $\widetilde{O}\left(\epsilon^{-(\mathcal{D}_s+\mathcal{D}_a+2)}\right)$ number of samples.
\end{theorem}

\begin{proof} (\textit{sketch})
Here we briefly state the key steps of our proof. Let $T_{\epsilon}$ be the time step such that learned $Q$ function under Hamiltonian $Q$-Learning is $\epsilon$ optimal. Then number of samples required by Hamiltonian $Q$-Learning to learn an $\epsilon$ optimal $Q$ function can be given as $\sum_{t=1}^{T_{\epsilon}} |\Omega_t || \mathcal{H}_t |.$ We first prove results on the sample size $|\Omega_t |$ required to bound the error incurred due to matrix completion. Then we prove results on the sample size $|\Omega_t |$ required to bound the error incurred by approximating the expectation of next state using HMC samples. Final result follows from combining aforementioned results with convergence and boundedness results obtained in Theorem \ref{thm:converge} and \ref{thm:convergeMat}. A detailed proof of Theorem~\ref{thm:sampleMat} is given in Supplementary Material.
\end{proof}
%

\section{Experiments}
We illustrate convergence and sample efficiency of Hamiltonian $Q$-Learning using a high-dimensional system and four benchmark control tasks. Recall that when $Q$ function is Lipschitz convergence in Frobenius norm implies convergence in infinity norm; therefore, we used the Frobenius norm of the difference between the learned $Q$ function and optimal $Q^*$ to illustrate that Hamiltonian $Q$-Learning converges to at $\epsilon$-optimal $Q$ function.
\subsection{Empirical Evaluation for a High-Dimensional System} 
\paragraph{Experimental setup for a double pendulum on a cart:} 
%
By letting $x,\dot{x}$ denote the position and velocity of the cart and $\theta_1, \theta_2, \dot{\theta}_1, \dot{\theta}_2$ denote the joint angles and angular velocities of the poles, we define the 6-dimensional state of the cart-pole system as: $s = (x, \dot{x}, \theta_1, \dot{\theta}_1, \theta_2, \dot{\theta}_2)$ where $x\in [-2.4,2.4]$, $\dot{x}\in [-3.5,3.5]$, and $\theta_i \in [-\pi,\pi]$, $\dot{\theta}_i\in [-3.0,3.0]$ for $i=1,2$. Also, we define the range of the scalar action as $a\in [-10,10]$. Then each state space dimension is discretized into 5 distinct values and the action space into 10 distinct values. This leads to a $Q$ matrix of size $15625\times 10$. We consider that the probabilistic state transition is governed by (\ref{eq:stateTrans}) with a $\Sigma$ which ensures that the range of the state space along direction $i$ approximately equals to $6\sqrt{\Sigma_i}$. To stabilize the pendulum to an upright position, we define the reward function as $r(s,a)=\cos^4(15\theta_1)+\cos^4(15\theta_2)$. After initializing the $Q$ matrix using randomly chosen values from $[0,2]$, we sample state-action pairs with probability $p=0.2$ at each iteration. Please refer Supplementary Material for additional details.
%
\begin{figure*}[t!]
    \centering
    \subfloat[
    \scriptsize{$Q$ Function Convergence}
    ]{{\includegraphics[width=0.27\textwidth]{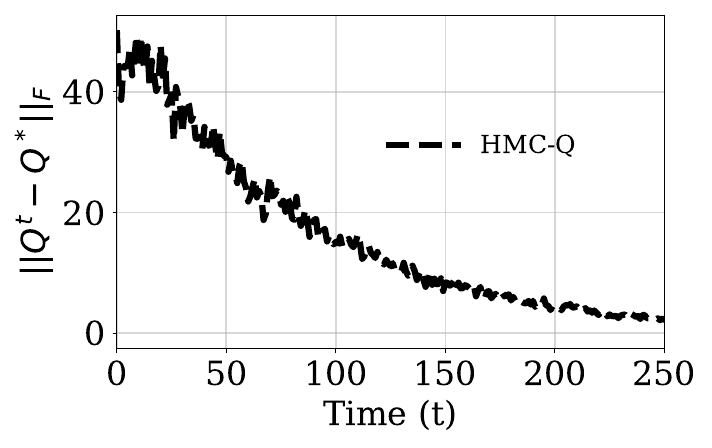}}}%
    \subfloat[
    \scriptsize{Exhaustive Sampling}
    ]
    {{\includegraphics[width=0.22\textwidth]{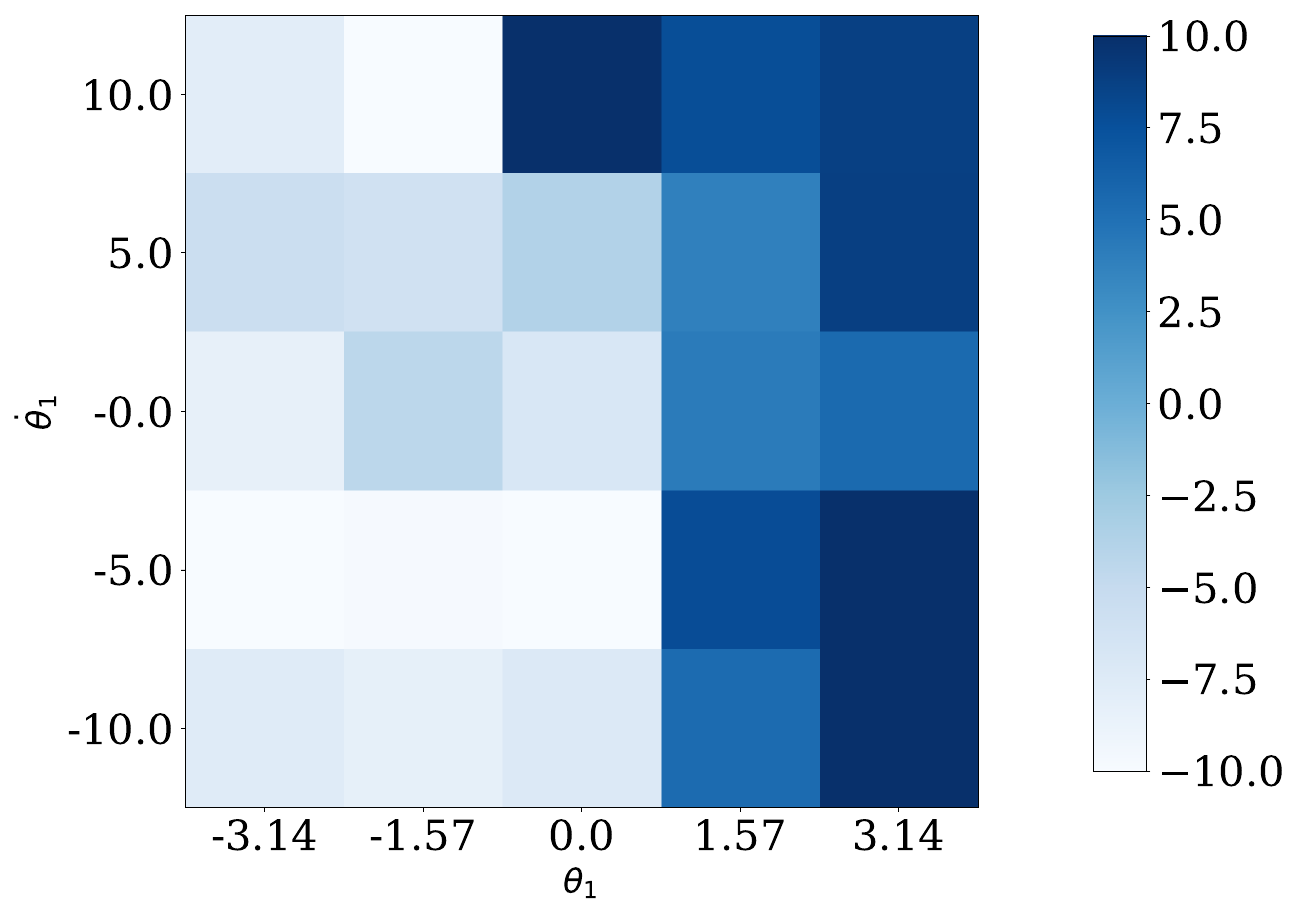}}}%
    \subfloat[
    \scriptsize{HMC Sampling}
    ]{{\includegraphics[width=0.22\textwidth]{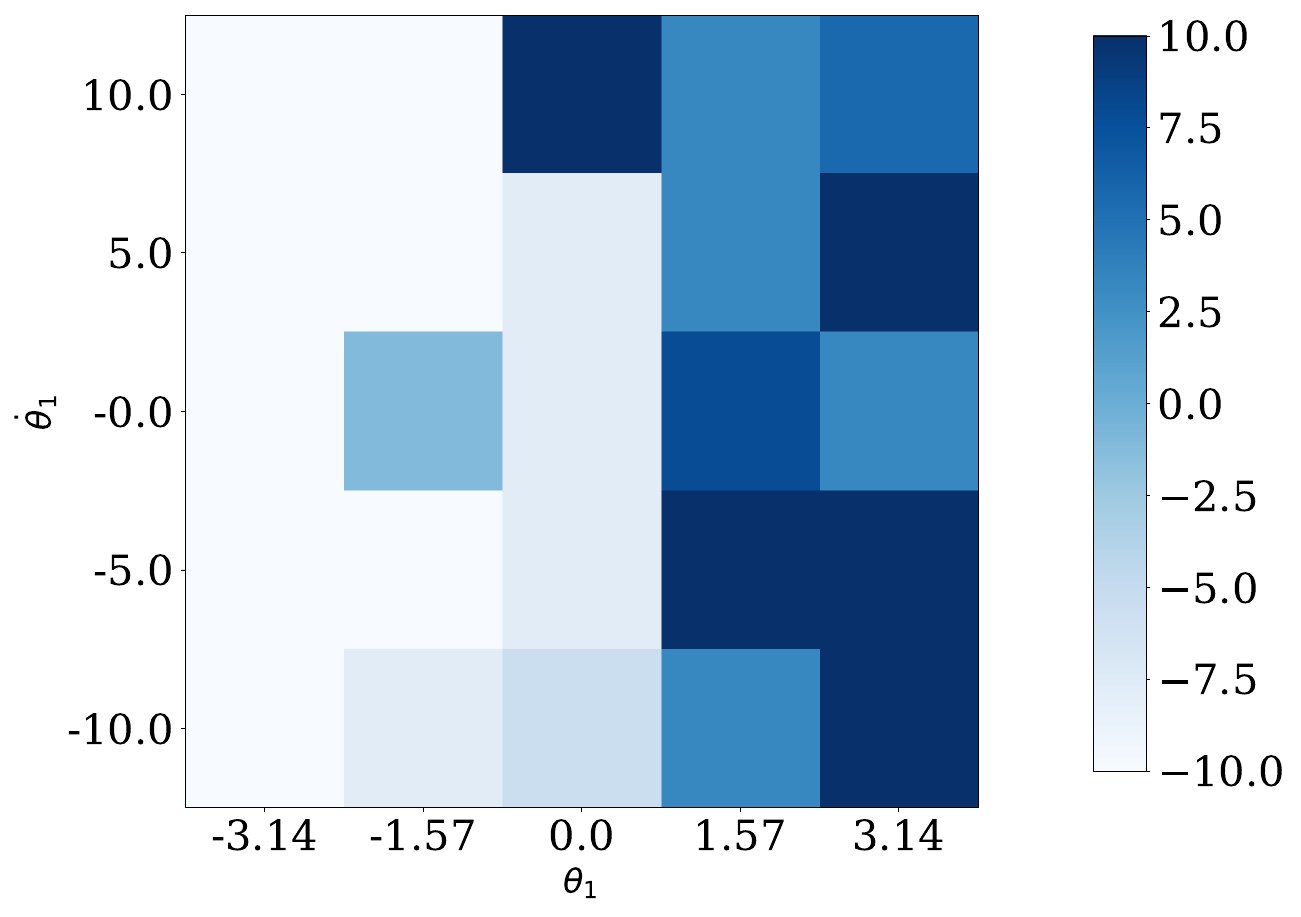}}}%
    \subfloat[
    \scriptsize{Sample efficiency}
    ]{{\includegraphics[width=0.28\textwidth]{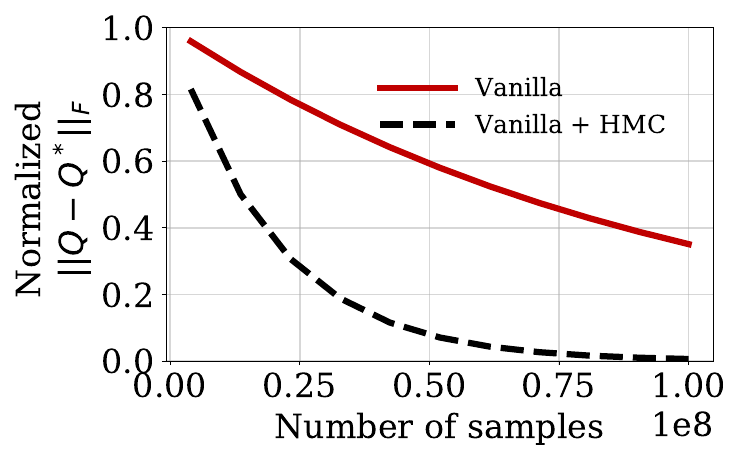}}}%
    \caption{\small{Figure \ref{fig:cartpole}(a) illustrates convergence of the $Q$ function learned via Hamiltonian $Q$-Learning to an $\epsilon$-optimal $Q$ function. Figure \ref{fig:cartpole}(b) and \ref{fig:cartpole}(c) show policy heat maps for $Q$-Learning with exhaustive sampling and Hamiltonian $Q$-Learning, respectively $(x=-1.2, \dot{x}=1.75, \theta_2=\pi/4,\dot{\theta}_2=1.5)$. Figure \ref{fig:cartpole}(d) shows the change in the normalized value of the Frobenius norm with the number of samples, for both exhaustive sampling and Hamiltonian $Q$-Learning for vanilla $Q$-Learning.}}%
    \label{fig:cartpole}%
\end{figure*}
\begin{figure*}[h!]
    \centering
    \includegraphics[width=0.85\textwidth]{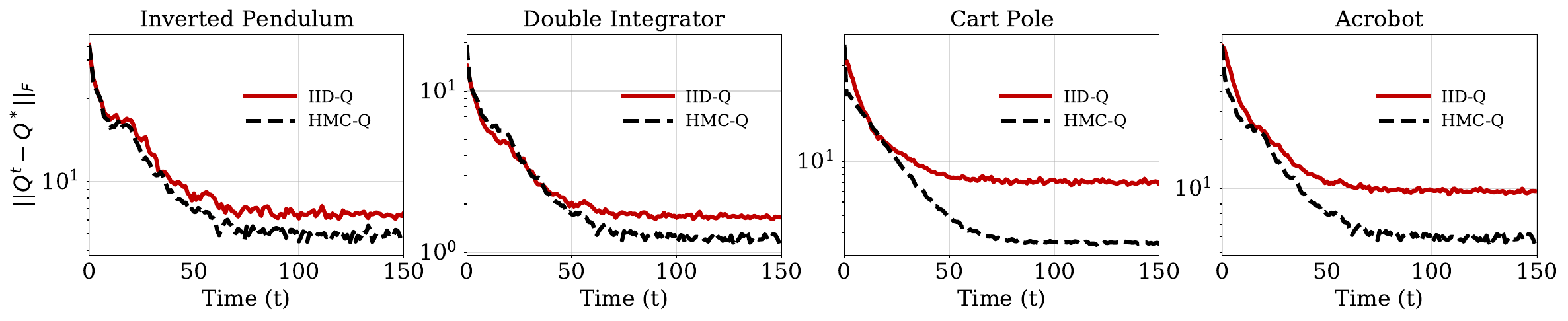}
    \vspace{-0.5em}
    \caption{\small{ A comparison of convergence of $Q$ function with Hamiltonian $Q$-Learning and $Q$-Learning with IID sampling.}}
    \label{fig:convergence}
    \vspace{1em}
\end{figure*}
\paragraph{Results:}
Figure~\ref{fig:cartpole}(a) shows the change in the Frobenius norm of the difference between the learned $Q$ function and optimal $Q^*$, thereby illustrating that Hamiltonian $Q$-Learning converges to an $\epsilon$ optimal $Q$ function. Note that under exhaustive sampling we use 15625 samples for each update. However, Hamiltonian $Q$-Learning uses only 200 samples for each update. As it is difficult to visualize policy heat maps for a 6-dimensional state space, we show results for the first two dimensions (i.e., $\theta_1$ and $\dot{\theta}_1$) while keeping the rest fixed (i.e., $\theta_2=0$, $\dot{\theta}_2=0$, $x=-1.2$, and $\dot{x}=3.5$). The heat maps shown in Figures \ref{fig:cartpole}(b) and \ref{fig:cartpole}(c) illustrate that the policy heat map for Hamiltonian $Q$-Learning is close to the one from $Q$-Learning with exhaustive sampling. We also show that the sample efficiency of $Q$-Learning can be significantly improved by incorporating Hamiltonian $Q$-Learning. Figure \ref{fig:cartpole}(d) shows how normalized Frobenius norm of the difference, i.e., Frobenius norm of the difference normalized by its maximum value, between the learned $Q$ function and the optimal $Q^*$ varies with increase in the number of samples. The solid red line shows the accuracy for exhaustive sampling and the dashed black line shows the same for Hamiltonian $Q$-Learning. These results show that Hamiltonian $Q$-Learning converges to an $\epsilon$ optimal $Q$ function with significantly fewer samples than exhaustive sampling.

\subsection{Empirical Evaluation for Low Dimensional Systems}
%
\paragraph{Experimental setup:} Here we investigate the applicability of Hamiltonian $Q$-Learning in low dimensional spaces where IID samples are available, and compare its performance against state-of-the-art algorithms on four benchmark control tasks (inverted pendulum, double integrator, cartpole, and acrobot). Among these four control tasks, the dynamics of inverted pendulum and double integrator evolve on a 2-dimensional state space, whereas cartpole and acrobot are defined on a 4-dimensional state space.  We discretize each state space dimension of inverted pendulum and double integrator into 25 distinct values, and each state space dimension of cartpole and acrobot into 5 distinct values. The action variable associated with all four control tasks is scalar, and we discretize each action space into 10 distinct values. This leads to a $Q$ matrix of size $625\times 10$. Please refer Supplementary Material for additional details about the experimental setup.
\begin{figure*}[!t]
    \centering
    \includegraphics[width=0.85\textwidth]{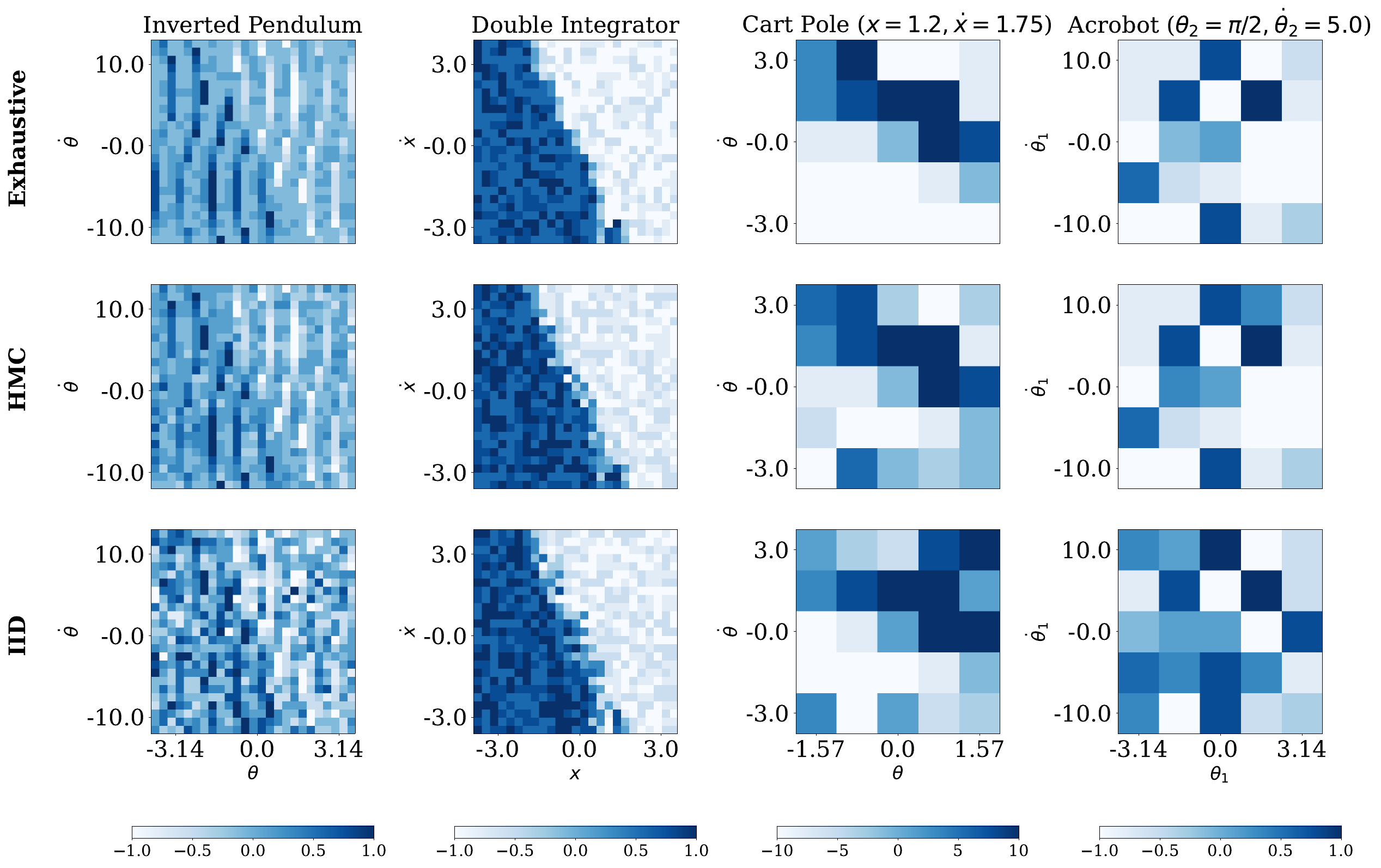}
    \caption{\small{Policy heatmaps for $Q$-Learning with exhaustive sampling, Hamiltonian $Q$-Leaning and IID sampling. The color in each cell corresponds to the value of optimal action at the corresponding state.}}
    \label{fig:heatmap}
\end{figure*}
\begin{figure*}[!h]
    \centering
    \includegraphics[width=0.85\textwidth]{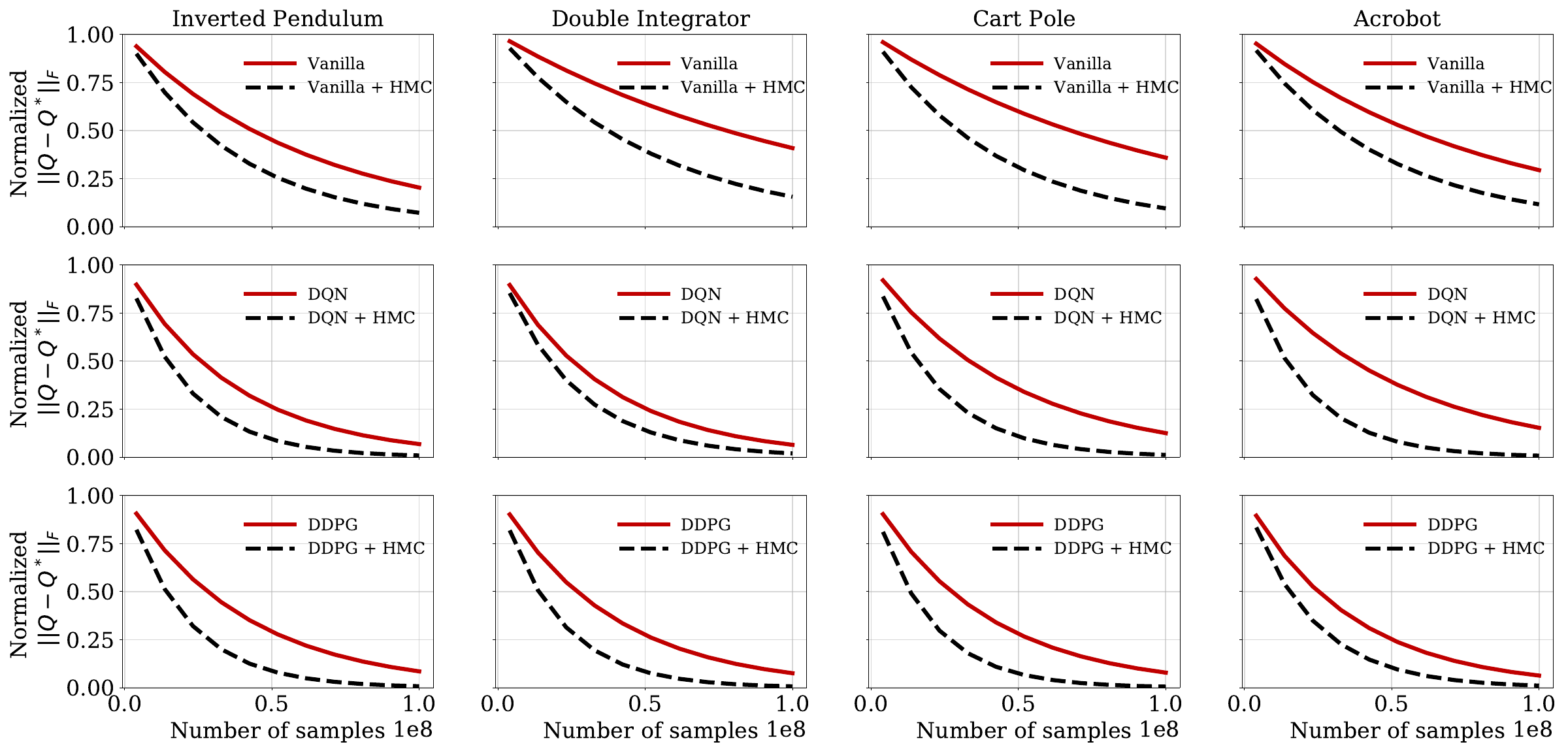}
    \caption{\small{Normalized mean square error, i.e. mean square error divided by it maximum, vs number of samples of $Q$ function with exhaustive sampling and HMC sampling for vanilla $Q$-Learning, DQN and DDPG. Red solid curve corresponds to HMC sampling and back dotted curve corresponds to HMC sampling.}}
    \label{fig:Sample}
    \vspace{1.0em}
\end{figure*}
\paragraph{Results:} Figure~\ref{fig:convergence} shows that Frobenius norm of the difference between the learned $Q$ function and optimal $Q^*$ can achieve a much lower value when HMC samples are used instead of IID samples. This illustrates that Hamiltonian $Q$-Learning achieves better convergence than $Q$-Learning with IID sampling. Note that, under exhaustive sampling we use 625 samples for each update, whereas learning with IID sampling and Hamiltonian $Q$-Learning require only 100 samples for each update. Figure \ref{fig:heatmap} shows policy heatmaps for $Q$-Learning with exhaustive sampling, Hamiltonian $Q$-Learning and $Q$-Learning with IID sampling. Our results show that the policy heatmaps associated from Hamitonian $Q$-Learning are closer to policy heatmaps obtained from $Q$-Learning with exhaustive sampling. Figure \ref{fig:Sample} illustrates how normalized Frobenius norm of the difference between the learned $Q$ function and the optimal $Q^*$ varies with increase in the number of samples. The solid red lines correspond to exhaustive sampling and the dashed black lines correspond to Hamiltonian $Q$-Learning. These results show that Hamiltonian $Q$-Learning can achieve the same level of accuracy with significantly fewer samples.
%
%
%
%
%
\section{Discussion and Conclusion}
In this paper we have introduced \textit{Hamiltonian Q-Learning}, a new model-free RL framework that can be utilized to obtain optimal policies in high-dimensional spaces, where obtaining IID samples is impractical.  
We show, both theoretically and empirically, that the proposed approach can learn accurate estimates of the optimal $Q$ function with much less numbr of samples compared to exhaustive sampling. Further, we illustrated that Hamiltonian Q-Learning can be used to improve sample efficiency of state-of-the-art algorithms in low dimensional spaces also. 
By building upon this aspect, future works will investigate how HMC sampling based methods can improve sample efficiency in multi-agent Q-learning, a system naturally very high-dimensions, with agents coupled through both action and reward.

\newpage
{
\small
\bibliographystyle{IEEEtran}
\bibliography{Manuscript_HMC}
}

\clearpage
\pagenumbering{arabic}
\renewcommand*{\thepage}{A\arabic{page}}
\setcounter{page}{1}
\input{Appendix_Paper}

\end{document}

%% file: math_commands.tex
\usepackage{amsmath,amsfonts,bm}

\def\eqref#1{equation~\ref{#1}}

\def\1{\bm{1}}

\DeclareMathAlphabet{\mathsfit}{\encodingdefault}{\sfdefault}{m}{sl}
\SetMathAlphabet{\mathsfit}{bold}{\encodingdefault}{\sfdefault}{bx}{n}

\DeclareMathOperator*{\argmin}{arg\,min}

%% file: Appendix_Paper.tex
\appendix
\setcounter{equation}{0}
\renewcommand{\theequation}{S.\arabic{equation}}

\textbf{\LARGE Supplementary Material}

\section{Convergence and Boundedness Results}
We proceed to prove theorem by stating convergence properties for HMC as follows. In the initial sampling stage, starting from the initial position Markov chain converges towards to the typical set. In the next stage Markov chain quickly traverse the typical set and improves the estimate by removing the bias. In the last stage Markov chain refine the exploration of typical the typical set provide improved estimates. The number of samples taken during the last stage is referred as effective sample size. 
\subsection{Proof of Theorem \ref{thm:converge}} \label{sec:proofofthHMC}
\setcounter{theorem}{0}
\begin{theorem}\label{thm:convergeapp}
Let $\EuScript{T}$ be an optimality operator under HMC given as $(\EuScript{T}Q)(s^{\prime},a^{\prime}) = r(s^{\prime},a^{\prime}) + \frac{\gamma}{|\mathcal{H}|}\sum_{s\in \mathcal{H}}\max _{a}Q(s,a), \:\:\forall (s^{\prime},a^{\prime})\in \mathcal{S}\times \mathcal{A},$ where $\mathcal{H}$ is a subset of next states sampled using HMC from the target distribution given in (\ref{eq:targetDis}). Then, under update rule (\ref{eq:HMCUpdate}) and for any given $\epsilon \geq 0$, there exists $n_{\mathcal{H}},t^{\prime}>0$ such that $\|Q^{t}-Q^*\|_{\infty}\leq \epsilon$ $\forall t\geq t^{\prime}$.
\end{theorem}
\begin{proofoftheorem}
\normalfont
 Let $\bar{Q}^{t}(s,a)=\frac{1}{n_{\mathcal{H}}}\max_a Q^t(s,a), \forall (s,a)\in \mathcal{S}\times \mathcal{A}.$ Here we consider $n_{\mathcal{H}}$ to be the effective number of samples. Let $\mathbb{E}_{\mathcal{P}}Q^t,\mathbf{Var}_{\mathcal{P}}Q^t$ be the expectation and covariance of $Q^t$ with respect to the target distribution. From Central Limit Theorem for HMC we have
 \begin{align*}
\bar{Q}^{t}\sim\mathcal{N}\left(\mathbb{E}_{\mathcal{P}}Q^t,\sqrt{\frac{\mathbf{Var}_{\mathcal{P}}Q^t}{n_{\mathcal{H}}}}\right).
\end{align*}
Since $Q$ function does not decay fast we provide a proof for the case where $Q^t$ is $C$-Lipschitz. 
From Theorem 6.5 in \cite{holmes2014curvature} we have that, there exists a $c_0>0$ such that
\begin{align}
||\bar{Q}^t-\mathbb{E}_{\mathcal{P}}Q^t||\leq c_0.\label{eq:HMCb}
\end{align}
Recall that Bellman optimality operator $\mathcal{T}$ is a contraction mapping. Thus from triangle inequality we have
\begin{align*}
  \Big |\Big | \EuScript{T}Q_1-\EuScript{T}Q_2\Big |\Big |_{\infty}\leq \max_{s^{\prime},a^{\prime}}   \Big |\Big | r(s^{\prime},a^{\prime})+ \frac{\gamma}{|\mathcal{H}_1|}\sum_{s\in \mathcal{S}}\max _{a}Q_1(s,a)\\
  -r(s^{\prime},a^{\prime})-\frac{\gamma}{|\mathcal{H}_2|}\sum_{s\in \mathcal{S}}\max _{a}Q_2(s,a)\Big |\Big |\\
  \leq \max_{s^{\prime},a^{\prime}}\Big |\Big | \frac{\gamma}{|\mathcal{H}_1|}\sum_{s\in \mathcal{S}}\max _{a}Q_1(s,a)-\frac{\gamma}{|\mathcal{H}_2|}\sum_{s\in \mathcal{S}}\max _{a}Q_2(s,a)\Big |\Big |
\end{align*}
Let $|\mathcal{H}_1|=|\mathcal{H}_2|=n_{\mathcal{H}}.$ Then using triangle inequality we have
\begin{align*}
 \Big |\Big | \EuScript{T}Q_1-\EuScript{T}Q_2\Big |\Big |_{\infty}\leq \max_{s^{\prime},a^{\prime}}  \gamma \left[\Big |\Big | \bar{Q}_1-\mathbb{E}_{\mathcal{P}}Q_1\Big |\Big |+ \Big |\Big |\bar{Q}_2-\mathbb{E}_{\mathcal{P}}Q_2\Big |\Big |\right]+\max_{s^{\prime},a^{\prime}}  \gamma\Big |\Big |\mathbb{E}_{\mathcal{P}}Q_1-\mathbb{E}_{\mathcal{P}}Q_2\Big |\Big |
\end{align*}
Since $Q$ function almost surely converge under exhaustive sampling we have
\begin{align}
\max_{s^{\prime},a^{\prime}}  \gamma\Big |\Big |\mathbb{E}_{\mathcal{P}}Q_1-\mathbb{E}_{\mathcal{P}}Q_2\Big |\Big |\leq \gamma \Big |\Big |Q_1-Q_2\Big |\Big |_{\infty}\label{eq:contrc}
\end{align}
From (\ref{eq:HMCb}) and (\ref{eq:contrc}) we have after $t$ time steps
\begin{align*}
  \Big |\Big | \EuScript{T}Q_1-\EuScript{T}Q_2\Big |\Big |_{\infty}\leq 2c_0+\gamma \Big |\Big |Q_1-Q_2\Big |\Big |_{\infty} 
\end{align*}
Let $R_{max}$ and $R_{min}$ be the maximum and minimum reward values. Then we have that
\begin{align*}
\Big |\Big |Q_1-Q_2\Big |\Big |_{\infty} \leq \frac{\gamma}{1-\gamma}R_{max}-R_{min}.   
\end{align*}
Thus for any $\epsilon\geq$ by choosing a $\gamma$ such there exists a $t^{\prime}$ such that $\forall t\geq t^{\prime}$ 
\begin{align*}
\|Q^{t}-Q^*\|_{\infty}\leq \epsilon 
\end{align*}
This concludes the proof of Theorem \ref{thm:converge}.
\qedsymbol \end{proofoftheorem}
\subsection{Proof of Theorem \ref{thm:convergeMat}} \label{sec:proofofthMat}
\begin{theorem}
Let $Q^{t+1}_{\mathcal{E}}(s_t,a_t)=r(s_t,a_t)+\gamma\sum_{s\in \mathcal{S}}\mathbb{P}(s|s_t,a_t)\max _{a}Q_{\mathcal{E}}^t(s,a), \forall (s_t,a_t)\in \mathcal{S}\times \mathcal{A}$ be the update rule under exhaustive sampling, and $Q^t$ be the $Q$ function updated according to Hamiltonian $Q$-Learning, i.e. by (\ref{eq:HMCSub})-(\ref{eq:HMCMC}). Then, for any given $\tilde{\epsilon} \geq 0$, there exists $n_{\mathcal{H}},t^{\prime}>0$, such that $\|Q^{t}-Q_{\mathcal{E}}^t\|_{\infty} \leq \tilde{\epsilon}$ $\forall t\geq t^{\prime}$.
\end{theorem}
\begin{proofoftheorem}
\normalfont
Note that at each time step we attempt to recover the matrix $Q^{t}_{\mathcal{E}},$ i.e., $Q$ function time $t$ under exhaustive sampling though a matrix completion method starting from $\widehat{Q}^{t}$, which is the $Q$ updated function at time $t$ using Hamiltonian $Q$-Learning.  From Theorem 4 in \cite{chen2018harnessing} we have that $\forall t\geq t^{\prime}$ there exists some constant $\delta>0$ such that when the updated $Q$ function a $\widehat{Q}^{t}$ satisfy 
\begin{align*}
\Big |\Big |\widehat{Q}^{t}-Q^{t}_{\EuScript{E}}\Big |\Big |_{\infty}\leq c
\end{align*}
where $c$ is some positive constant then reconstructed (completed) matrix $Q^{t}$ satiesfies
\begin{align}
\Big | \Big |Q^{t}-Q^{t}_{\EuScript{E}}\Big |\Big |_{\infty}\leq \delta \Big |\widehat{Q}^{t}-Q^{t}_{\EuScript{E}}\Big |\Big |_{\infty}\label{eq:MCconv}
\end{align}
for some $\delta>0.$ This implies that when the initial matrix used for matrix completion is sufficiently close to the matrix we are trying to recover matrix completion iterations converge to a global optimum. 
From the result of Theorem \ref{thm:converge} we have for any given $\epsilon\geq 0$, there exists $n_{\EuScript{H}},t^{\prime}>0$ such that $\forall t\geq t^{\prime}$
\begin{align}
\Big | \Big | \widehat{Q}^t-Q^* \Big | \Big |\leq {\epsilon}\label{eq:Gconv}
\end{align}
Recall that under the update equation  $Q^{t+1}_{\mathcal{E}}(s_t,a_t)=r(s_t,a_t)+\gamma\sum_{s\in \mathcal{S}}\max _{a}Q_{\mathcal{E}}^t(s,a), \forall (s_t,a_t)\in \mathcal{S}\times \mathcal{A}$ we have that $Q_{\mathcal{E}}$ almost surely converge to the optimal $Q^*.$ Thus there exists a $t^{\dagger}$ such that $\forall t\geq t^{\dagger}$
\begin{align*}
\Big | \Big | Q^t_{\EuScript{E}}-Q^* \Big | \Big |\leq {\epsilon}
\end{align*}
Let $t^{\ddagger}=\max\{{t^{\dagger},t^{\prime}}\}.$ Then from triangle inequality we have that
\begin{align*}
\Big | \Big | \widehat{Q}^t-Q^t_{\EuScript{E}} \Big | \Big |\leq \Big | \Big | \widehat{Q}^t-Q^* \Big | \Big |+\Big | \Big | Q^t_{\EuScript{E}}-Q^* \Big | \Big |\leq 2\epsilon.
\end{align*} Thus from (\ref{eq:MCconv}) we have that
\begin{align*}
\Big | \Big |Q^{t}-Q^{t}_{\EuScript{E}}\Big |\Big |_{\infty}\leq 2\delta\epsilon
\end{align*}

This concludes the proof of Theorem \ref{thm:convergeMat}.
\qedsymbol  \end{proofoftheorem}
\section{Sampling Complexity}\label{sec:SampleComplex}
In this section we provide theoretical results on sampling complexity of Hamiltonian $Q$-Learning. For brevity of notation we define $\EuScript{M}Q(s)=\max_a Q(s,a).$ Note that we have the following regularity conditions on the MDP studied in this paper.

{\bf{Regularity Conditions}}
\begin{enumerate}
    \item Spaces $\mathcal{S}$ and $\mathcal{A}$ (state space and action space) are compact subsets of $\mathds{R}^{\mathcal{D}_s}$ and $\mathds{R}^{\mathcal{D}_a}$ respectively.
    \item All the rewards are bounded such that $r(s,a)\in [R_{\min},R_{\max}],$ for all $(s,a)\in \mathcal{S}\times \mathcal{A}.$
    \item The optimal $Q^*$ is $C$-Lipschitz such that 
    \begin{align*}
    \Big | Q^*(s,a)-Q^*(s^{\prime},a^{\prime})\Big |\leq C\left( ||s-s^{\prime}||_{F}+||a-a^{\prime}||_{F}\right)
    \end{align*}
\end{enumerate}

Now we prove some useful lemmas for proving sampling complexity of Hamiltonian $Q$-Learning
\begin{lemma}\label{lem:NumMatCom}
For some constant $c_1$, if
\begin{align*}
|\Omega_t | \geq c_1\frac{\max\Big \{ |\mathcal{S}|^2,|\mathcal{A}|^2\Big \}|\mathcal{S}||\mathcal{A}|\mathcal{D}_s\mathcal{D}_a}{\log \left(\mathcal{D}_s+\mathcal{D}_a\right)}
\end{align*}
with $\Big |\Big | \widehat{Q}^t(s,a)-Q^*(s,a)\Big |\Big |_{\infty}\leq \epsilon$ then there exists a constant $c_2$ such that
\begin{align*}
\Big |\Big | {Q}^t(s,a)-Q^*(s,a)\Big |\Big |_{\infty}\leq c_2\epsilon
\end{align*}
\end{lemma}

\begin{proofoflemma}
\normalfont Recall that in order to complete a low rank matrix using matrix estimation methods, the matrix can not be sparse. This condition can be formulated using the notion of incoherence. Let $Q_{}$ be a matrix of rank $r_Q$ with the singular value decomposition $Q=U\Sigma V^T.$ Let $T_Q$ be the orthogonal projection of $Q\in \mathds{R}^{|\mathcal{S}|\times |\mathcal{A}|}$ to its column space. Then incoherence parameter of $\phi(Q)$ can be give as
\begin{align*}
\phi(Q) = \max \Big \{\frac{|\mathcal{S}|}{r_Q}\max _{1\leq i\leq |\mathcal{S}|} ||T_U\mathbf{e}_i||_F^2 ,\frac{|\mathcal{A}|}{r_Q}\max _{1\leq i\leq |\mathcal{A}|} ||T_U\mathbf{e}_i||_F^2\Big \}
\end{align*}
where $\mathbf{e}_i$ are the standard basis vectors.
Recall that $Q^t$ is the matrix generated in matrix completion phase from $\widehat{Q}.$  
From Theorem 4 in \cite{chen2018harnessing} we have that for some constant $C_1$ if a fraction of $p$ elements are observed from the matrix such that
\begin{align*}
p\geq  C_1 \frac{\phi_t^2r_Q^2\mathcal{D}_s\mathcal{D}_a}{\log \left(\mathcal{D}_s+\mathcal{D}_a\right)}
\end{align*}
where $\phi_t$ is the coherence parameter of $Q^t$ then with probability at least $1-C_2(\mathcal{D}_s+\mathcal{D}_a)^{-1}$ for some constant $C_2$
with $\Big |\Big | \widehat{Q}^t(s,a)-Q^*(s,a)\Big |\Big |_{\infty}\leq \epsilon$ there exists a constant $c_2$ such that
\begin{align*}
\Big |\Big | {Q}^t(s,a)-Q^*(s,a)\Big |\Big |_{\infty}\leq c_2\epsilon
\end{align*}
Note that $p\approx \frac{|\Omega_t|}{|\mathcal{S}||\mathcal{A}|}.$ Further we have for some constant $c_3$
\begin{align*}
\frac{\phi_t^2r_Q^2\mathcal{D}_s\mathcal{D}_a}{\log \left(\mathcal{D}_s+\mathcal{D}_a\right)}=c_3\frac{\max\Big \{ |\mathcal{S}|^2,|\mathcal{A}|^2\Big \}\mathcal{D}_s\mathcal{D}_a}{\log \left(\mathcal{D}_s+\mathcal{D}_a\right)}
\end{align*}
Thus it follows that for some constant $c_1$ if 
\begin{align*}
|\Omega_t|=c_1\frac{\max\Big \{ |\mathcal{S}|^2,|\mathcal{A}|^2\Big \}|\mathcal{S}||\mathcal{A}|\mathcal{D}_s\mathcal{D}_a}{\log \left(\mathcal{D}_s+\mathcal{D}_a\right)}
\end{align*}
with $\Big |\Big | \widehat{Q}^t(s,a)-Q^*(s,a)\Big |\Big |_{\infty}\leq \epsilon$ then there exists a constant $c_2$ such that
\begin{align*}
\Big |\Big | {Q}^t(s,a)-Q^*(s,a)\Big |\Big |_{\infty}\leq c_2\epsilon
\end{align*}
This concludes the proof of Lemma \ref{lem:NumMatCom}. \qedsymbol
\end{proofoflemma}
\begin{lemma}\label{lem:bound}
Let $1-\xi$ be the spectral gap of Markov chain under Hamiltonian sampling where $\xi \in [0,1].$ Let $\Delta R=R_{\max}-R_{\min}$ be the maximum reward gap. Then $\forall (s^{\prime},a^{\prime})\in \mathcal{S}\times \mathcal{A}$ we have that
\begin{align*}
|\widehat{Q}(s^{\prime},a^{\prime})-Q^*(s^{\prime},a^{\prime})\Big |\leq \frac{\gamma^2}{1-\gamma}\Delta R+\sqrt{\frac{1+\xi}{1-\xi}\frac{2}{|\mathcal{H}|}\left(\frac{\gamma R_{\max}}{1-\gamma}\right)^2\log\left(\frac{2}{\delta}\right)}. 
\end{align*}
with at least probability $1-\delta.$
\end{lemma}
\begin{proofoflemma}
\normalfont
Let $\widehat{Q}(s^{\prime},a^{\prime})=r(s^{\prime},a^{\prime})+\frac{\gamma}{|\mathcal{H}|}\sum_{s\in \mathcal{H}}\max_a Q(s,a).$ Recall that $\EuScript{M}Q(s)=\max_aQ(s,a).$ Then we have that $\widehat{Q}(s^{\prime},a^{\prime})=r(s^{\prime},a^{\prime})+\frac{\gamma}{|\mathcal
{H}|}\sum_{s\in \mathcal{H}}\EuScript{M}Q(s).$ Then it follows that
\begin{align}
|\widehat{Q}(s^{\prime},a^{\prime})-Q^*(s^{\prime},a^{\prime})\Big |&= \Big |r(s^{\prime},a^{\prime})+\frac{\gamma}{|\mathcal
{H}|}\sum_{s\in \mathcal{H}}\EuScript{M}Q(s)-r(s^{\prime},a^{\prime})-\gamma \mathbb{E}_{\mathcal{P}}\EuScript{M}Q^*(s)\Big |\nonumber \\
&=\Big |\frac{\gamma}{|\mathcal
{H}|}\sum_{i=1}^{|\mathcal{H}|}\EuScript{M}Q(s_i)-\gamma \mathbb{E}_{\mathcal{P}}\EuScript{M}Q^*(s)\Big |\nonumber \\
&= \Big |\frac{\gamma}{|\mathcal
{H}|}\sum_{i=1}^{|\mathcal{H}|}\EuScript{M}Q(s_i)-\frac{\gamma}{|\mathcal
{H}|}\sum_{i=1}^{|\mathcal{H}|}\EuScript{M}Q^*(s_i)\Big |\nonumber \\
&\:\:\:\:\:\:\:\:\:+\Big  |\frac{\gamma}{|\mathcal
{H}|}\sum_{i=1}^{|\mathcal{H}|}\EuScript{M}Q^*(s_i)-\gamma\mathbb{E}_{\mathcal{P}}\EuScript{M}Q^*(s)\Big |\label{eq:IntDecomp}
\end{align}
Recall that all the rewards are bounded such that $r(s,a)\in [R_{\min},R_{\max}],$ for all $(s,a)\in \mathcal{S}\times \mathcal{A}.$ Thus for all $ s,a $ we have that $\EuScript{M}Q(s)\leq \frac{\gamma}{1-\gamma}R_{\max}.$ Let $\Delta R=R_{\max}-R_{\min}.$ Then we have that
\begin{align}
\Big |\frac{\gamma}{|\mathcal
{H}|}\sum_{i=1}^{|\mathcal{H}|}\EuScript{M}Q(s_i)-\frac{\gamma}{|\mathcal
{H}|}\sum_{i=1}^{|\mathcal{H}|}\EuScript{M}Q^*(s_i)\Big |\leq \frac{\gamma^2}{1-\gamma}\Delta R.\label{eq:firstb}
\end{align}
Let $\xi\in [0,1]$ be a constant such that $1-\xi$ is the spectral gap of the Markov chain under Hamiltonian sampling. Then from \cite{fan2018hoeffding} we have that
\begin{align*}
\mathbb{P}\left(\frac{1}{|\mathcal
{H}|}\sum_{i=1}^{|\mathcal{H}|}\EuScript{M}Q^*(s_i)-\mathbb{E}_{\mathcal{P}}\EuScript{M}Q^*(s)\geq \vartheta \right)\leq \exp\left(-\frac{1-\xi}{1+\xi}\frac{|\mathcal{H}|\vartheta^2}{2R_{\max}^2}\left(\frac{1-\gamma}{\gamma}\right)^2\right)
\end{align*}
Let $\delta=\exp\left(-\frac{1-\xi}{1+\xi}\frac{|\mathcal{H}|\vartheta^2}{2R_{\max}^2}\left(\frac{1-\gamma}{\gamma}\right)^2\right).$ Then we have that
\begin{align*}
\vartheta=\sqrt{\frac{1+\xi}{1-\xi}\frac{2}{|\mathcal{H}|}\left(\frac{\gamma R_{\max}}{1-\gamma}\right)^2\log\left(\frac{2}{\delta}\right)}.
\end{align*}
Thus we see that
\begin{align}
\Big |\frac{1}{|\mathcal
{H}|}\sum_{i=1}^{|\mathcal{H}|}\EuScript{M}Q^*(s_i)-\mathbb{E}_{\mathcal{P}}\EuScript{M}Q^*(s) \Big | \leq \sqrt{\frac{1+\xi}{1-\xi}\frac{2}{|\mathcal{H}|}\left(\frac{\gamma R_{\max}}{1-\gamma}\right)^2\log\left(\frac{2}{\delta}\right)} \label{eq:secondB}
\end{align}
with at least probability $1-\delta.$ Thus it follows from equations (\ref{eq:IntDecomp}), (\ref{eq:firstb}) and (\ref{eq:secondB}) that
\begin{align*}
|\widehat{Q}(s^{\prime},a^{\prime})-Q^*(s^{\prime},a^{\prime})\Big |\leq \frac{\gamma^2}{1-\gamma}\Delta R+\sqrt{\frac{1+\xi}{1-\xi}\frac{2}{|\mathcal{H}|}\left(\frac{\gamma R_{\max}}{1-\gamma}\right)^2\log\left(\frac{2}{\delta}\right)}. 
\end{align*}
with at least probability $1-\delta.$ This concludes the proof of Lemma \ref{lem:bound}.  \qedsymbol 
\end{proofoflemma}
\begin{lemma}\label{lem:time Sum}
For all $(s,a)\in \mathcal{S}\times \mathcal{A}$ we have that
\begin{align*}
|{Q}^t(s,a)-Q^*(s,a)\Big |\leq 2c_1\frac{\gamma^2R_{\max}}{1-\gamma}
\end{align*}
with probability at least $1-\delta$
\end{lemma}

\begin{proofoflemma}
\normalfont
From Lemma \ref{lem:bound} and \cite{shah2020sample} we have that for all $(s,a)\in \Omega_t$ 
\begin{align}
|\widehat{Q}^t(s,a)-Q^*(s,a)\Big |\leq \frac{\gamma^2}{1-\gamma}\Delta R+\sqrt{\frac{1+\xi}{1-\xi}\frac{2}{|\mathcal{H}_t|}\left(\frac{\gamma R_{\max}}{1-\gamma}\right)^2\log\left(\frac{2|\Omega_t|T}{\delta}\right)}. \label{eq:Hatbound}
\end{align}
with probability at least $1-\frac{\delta}{T}.$ Thus we have that
\begin{align*}
|{Q}^t(s,a)-Q^*(s,a)\Big |\leq c_1\frac{\gamma^2}{1-\gamma}\Delta R+c_1\sqrt{\frac{1+\xi}{1-\xi}\frac{2}{|\mathcal{H}_t|}\left(\frac{\gamma R_{\max}}{1-\gamma}\right)^2\log\left(\frac{2|\Omega_t|T}{\delta}\right)}.
\end{align*}
with probability at least $1-\frac{\delta}{T}.$ Fro all $1\leq t \leq T$ letting 
\begin{align*}
|\mathcal{H}_t|=\frac{1+\xi}{1-\xi}\frac{2}{\gamma^2}\log\left(\frac{2|\Omega_t|T}{\delta}\right)
\end{align*}
we obtain
\begin{align*}
\frac{\gamma^2}{1-\gamma} R_{\max} \geq  \sqrt{\frac{1+\xi}{1-\xi}\frac{2}{|\mathcal{H}_t|}\left(\frac{\gamma R_{\max}}{1-\gamma}\right)^2\log\left(\frac{2|\Omega_t|T}{\delta}\right)}.
\end{align*}
Thus we have,
\begin{align*}
|{Q}^t(s,a)-Q^*(s,a)\Big |\leq 2c_1\frac{\gamma^2R_{\max}}{1-\gamma}
\end{align*}
with probability at least $1-\delta.$ Recall that $\forall (s,a)\in \mathcal{S}\times \mathcal{A}$ we have $\EuScript{M}Q(s,a)\leq \frac{\gamma \Delta R}{1-\gamma}.$ Thus this also proves that
\begin{align*}
|{Q}^t(s,a)-Q^*(s,a)\Big |\leq 2c_1\gamma |{Q}^{t-1}(s,a)-Q^*(s,a)\Big |
\end{align*}
This concludes the proof of Lemma \ref{lem:time Sum}. \qedsymbol 
\end{proofoflemma}

Now we proceed to prove the main theorem for sampling complexity as follows.

\begin{theorem}\label{thm:samplecomplex}
Let $\mathcal{D}_s,\mathcal{D}_a$ be the dimension of state space and action space respectively. Consider the Hamiltonian $Q$-Learning algorithm presented in Algorithm \ref{alg:HMCQAlg}. Under a suitable matrix completion method sampling complexity of the algorithm, $Q$ function converge to the family of $\epsilon$-optimal $Q$ functions with  $\widetilde{O}\left(\epsilon^{-(\mathcal{D}_s+\mathcal{D}_a+2)}\right)$ number of samples.
\end{theorem}
\begin{proofoftheorem}
\normalfont
Note that sample complexity of Hamiltonian $Q$-Learning can be given as
\begin{align*}
\sum_{t=1}^{T_{\epsilon}}|\Omega_t||\mathcal{H}_t|\leq T_{\epsilon}|\Omega_{T_{\epsilon}}||\mathcal{H}_{T_{\epsilon}}|
\end{align*}
Let $\beta^t$ be the discretization parameter at time $t$ and $T_{\epsilon}=\frac{\log \left(\frac{\gamma R_{\max}}{(1-\gamma)\epsilon}\right)}{\log\left(\frac{1}{2\gamma c_1}\right)}$. Then from Lemmas \ref{lem:NumMatCom}, \ref{lem:bound} and \ref{lem:time Sum} it follows that
\begin{align*}
\sum_{t=1}^{T_{\epsilon}}|\Omega_t||\mathcal{H}_t| = \widetilde{O}\left(\frac{1}{\epsilon^{\mathcal{D}_s+\mathcal{D}_a+2}}\right)
\end{align*}
This concludes the proof of Theorem \ref{thm:samplecomplex}. \qedsymbol
\end{proofoftheorem}

\section{Additional Experimental Details for Benchmark Control Tasks}\label{sec:appcontrol}
In this section we provide additional details related to the experimental results presented in this paper. 

\subsection{Experimental Setup}
We consider the case that state transition is stochastic due to system noise arise from model uncertainties. Following the conventional approach we model these parameter uncertainties and external disturbances using a multivariate Gaussian perturbation \cite{maithripala2016geometric, madhushani2017feedback, mcallister2017data}. For all the control tasks we consider the dynamic equations given in \cite{shah2020sample}. For all simulations we take 100 HMC samples during the update phase. We use trajectory length $L=100$ and step size $\delta l=0.02.$ We randomly initialize the $Q$ matrix using values between 0 and 1. 
\paragraph{Inverted Pendulum} 
Let $\theta$, $\dot{\theta}$ be the angle of the pendulum, respectively. Then, by letting $a$ denote the input torque applied to the pendulum, its dynamics can be expressed as
\begin{equation}
    \ddot{\theta} - \sin\theta + \dot{\theta} - a = 0.\label{eq:pend}
\end{equation}

The state space associated with the pendulum is 2-dimensional ($\EuScript{D}_s=2$) and any state $s\in \mathcal{S}$ is given by $s=(\theta,\dot{\theta})$. We define the range of state space as $\theta\in [-\pi,\pi]$ and $\dot{\theta}\in [-10, 10]$. We consider action space to be a 1-dimensional ($\EuScript{D}_a=1$) space such that $a\in [-1,1]$. We discretize each dimension in state space into 25 values and the action space into 10 values. This forms a $Q$ matrix of dimension $625\times 10$. 

Also, we consider the noise co-variance of the Gaussian perturbation to be $\Sigma = \text{diag}[0.868,1.550]$.

Let $s_t=({\theta}_t,\dot{\theta}_t)$ and $a_t$ be the state and the action at time $t$. Then the state transition probability kernel and corresponding target distribution can be given using (\ref{eq:stateTrans}) and (\ref{eq:targetDisNormal}), respectively, with mean $\mathbf{\mu}(s_t,a_t) = ({\theta}_t+ \dot{\theta}_t\tau,\dot{\theta}_t+\ddot{\theta}_t\tau)$, where  $\tau$ is the discretizated time interval and $\ddot{\theta}_t$ can be obtained from (\ref{eq:pend}) by substituting $\theta_t,\dot{\theta}_t,a_t,$
and co-variance $\Sigma(s_t,a_t) = \Sigma$.

As our goal is to stabilize the pendulum to the upright position (i.e. to $\theta=0$) while minimizing the amount of applied torque, we consider the reward function as follows
\begin{align*}
 r(\theta, \dot{\theta},a) = -0.1a^2 + \exp (\cos\theta -1).
\end{align*}

\paragraph{Double Integrator}
By letting $x$, $\dot{x}$, and $a$ denote the position, velocity, and input torque, respectively, we can express the system dynamics as
\begin{align}
    \ddot{x} &= a\label{eq:doubInt}.
\end{align}

State space of the double integrator is 2-dimensional ($\EuScript{D}_s=2$) and any state $s\in \mathcal{S}$ is given as $s=(x,\dot{x}).$ We define the range of state space as $x\in [-3,3]$ and $\dot{\theta}\in [-3, 3]$. We consider action space to be a 1-dimensional ($\EuScript{D}_a=1$) space such that $a\in [-1,1]$. We discretize each dimension in state space into 25 values and action space into 10 values. This forms a $Q$ matrix of dimension $625\times 10$.

Here we consider the noise co-variance of the Gaussian perturbation to be $\Sigma = \text{diag}[0.848, 0.848]$.

Let $s_t=(x_t,\dot{x}_t)$ and $a_t$ be the state and the action at time $t$. Then the state transition probability kernel and corresponding target distribution can be given using (\ref{eq:stateTrans}) and (\ref{eq:targetDisNormal}), respectively, with mean $\mathbf{\mu}(s_t,a_t) = (x_t+ \dot{x}_t\tau,\dot{x}_t+\ddot{x}_t\tau)$, where $\tau$ is the discretizated time interval and $\ddot{x}_t$ can be obtained from (\ref{eq:doubInt}) by substituting $x_t,\dot{x}_t,a_t,$
and co-variance $\Sigma(s_t,a_t) = \Sigma$.

We define the reward function as the quadratic cost
\begin{align*}
 r(x, \dot{x},a) = -\frac{1}{2}\left(x^2+\dot{x}^2\right).
\end{align*}
\paragraph{Cartpole}
Let $\theta$, $\dot{\theta}$ be the angle and angular velocity of the pole, respectively. Similarly, let $x$, $\dot{x}$ be the position and linear velocity of the cart, respectively. Then, by letting $a$ denote the control force applied to the cart, the dynamics of cart-pole system \cite{6313077} can be expressed as
\begin{equation}
\begin{aligned}
l\left(\frac{4}{3}(m+M) - m\cos^2 \theta\right) \ddot{\theta} + (a+ml \dot{\theta}^2\sin \theta)\cos \theta - (m+M)g\sin \theta &= 0
\\
(m+M)\ddot{x} - ml\left(\dot{\theta}^2\sin \theta-\ddot{\theta}\cos\theta\right) - a &= 0
\end{aligned}
\label{eq:cartpole}
\end{equation}
where, $m$, $M$, $l$ and $g$ represent the mass of the pole, mass of the cart, length of the pole and the gravitational acceleration, respectively.

State space of the cart-pole system is 4-dimensional ($\EuScript{D}_s=4$) and any state $s\in \mathcal{S}$ is given by $s=(\theta,\dot{\theta},x,\dot{x})$. We define the range of state space as $\theta \in [-pi/2,\pi/2], \dot{\theta}\in [-3.0,3.0],x\in [-2.4,2.4]$ and $\dot{x}\in [-3.5,3.5]$. We consider action space to be a 1-dimensional ($\EuScript{D}_a=1$) space such that $a\in [-10,10]$. We discretize each dimension in state space into 5 values and action space into 10 values. This forms a $Q$ matrix of dimensions $625\times 10$.

Although the differential equations (\ref{eq:cartpole}) governing the dynamics of the pendulum on a cart system are deterministic, uncertainty of the parameters and external disturbances to the system causes the cart pole to deviate from the defined dynamics leading to a stochastic state transition. Here we consider the co-variance of the Gaussian perturbation to be $\Sigma = \text{diag}[0.641,0.848,0.759,0.917]$.

Let $s_t=(\theta_t,\dot{\theta}_t,x_t,\dot{x}_t)$ and $a_t$ be the state and the action at time $t$. Then the state transition probability kernel and corresponding target distribution can be given using (\ref{eq:stateTrans}) and (\ref{eq:targetDisNormal}), respectively, with mean $\mathbf{\mu}(s_t,a_t) = (\theta_t+ \dot{\theta}_t\tau,\dot{\theta}_t+\ddot{\theta}_t\tau,x_t+\dot{x}_t\tau,\dot{x}_t+\ddot{x}_t\tau)$, where $\tau$ is the discretizated time interval and $\ddot{\theta_t},\ddot{x}_t$ can be obtained from (\ref{eq:cartpole}) by substituting $\theta_t,\dot{\theta}_t,a_t,$
and co-variance $\Sigma(s_t,a_t) = \Sigma$.

Our simulation results use the following value for the system parameters - $m=0.1 kg$, $M=1 kg$, $l=0.5 m$ and $g=9.8 ms^{-2}$. The goal is stabilizing the pole in upright position. Thus we consider the reward function
\begin{align*}
    r(\theta,\dot{\theta}, x,\dot{x},a)= \cos^4(15\theta)
\end{align*}

\paragraph{Acrobot}
Let $\theta_1$, $\dot{\theta}_1$ be the angle and angular velocity of the first pole, respectively. Similarly, let $\theta_2$, $\dot{\theta}_2$ be the angle and angular velocity of the first pole, respectively. Then, by letting $a$ denote the control torque applied to the second joint, dynamics of the acrobot can be expressed as
\begin{equation}
\begin{aligned}
\ddot{\theta}_2& = \frac{a+\frac{D_2}{D_1}\phi_1-m_2l_1l_{c2}\dot{\theta}_1^2\sin\theta_2-\phi_2}{m_2(l_2^2+l_{c2}^2)-\frac{D_2^2}{D_1}}
\\
\ddot{\theta}_1& = -\frac{D_2\ddot{\theta_2}+\phi_1}{D_1},
\end{aligned}
\label{eq:acrobot}
\end{equation}
where,
\begin{align*}
D_1& = m_1(l_1^2+l_{c1}^2)+m_2(l_1^2+l_2^2+l_{c2}^2+2l_1l_{c2}\cos\theta_2)\\
D_2& = m_2(l_2^2+l_{c2}^2+l_1l_{c2}\cos\theta_2)\\
\phi_2& = m_2l_{c2}g\sin (\theta_1+\theta_2)\\
\phi_1& = -m_2l_{c2}\dot{\theta}_2(\dot{\theta}_2+2\dot{\theta}_1)\sin\theta_2+(m_1l_{c1}+m_2l_1)g\sin\theta_1+phi_2,
\end{align*}
and $m_1$, $m_2$, $l_1$ $l_2$ and $g$ represent the mass of the poles, length of the poles, and the gravitational acceleration, respectively. We have used $l_{c1}=l_1/2$ and $l_{c2}=l_2/2$. Moreover, our simulation results use the following value for the system parameters: $m_1=m_2=0.1 kg$, $l_1=l_2=0.1 m$ and $g=9.8 ms^{-2}$.

State space of the acrobot is 4-dimensional ($\EuScript{D}_s=4$) and any state $s\in \mathcal{S}$ is given by $s = (\theta_1, \dot{\theta}_1, \theta_2, \dot{\theta}_2)$. We define the range of state space as $\theta_i \in [-\pi,\pi]$ and $\dot{\theta}_i\in [-10.0,10.0]$, $i=1,2$. We consider action space to be a 1-dimensional ($\EuScript{D}_a=1$) space such that $a\in [-1,1]$. We discretize each dimension in state space into 5 values and action space into 10 values. This forms a $Q$ matrix of dimensions $625\times 10$.

To incorporate the effects from uncertainty in the parameters and external disturbances to the system, we consider the co-variance of the Gaussian perturbation to the system be $\Sigma = \text{diag}[0.686,1.550,0.686,1.550]$.

Let $s_t=({\theta_1}_t,\dot{\theta_1}_t,{\theta_2}_t,\dot{\theta_2}_t)$ and $a_t$ be the state and the action at time $t$. Then the state transition probability kernel and corresponding target distribution can be given by (\ref{eq:stateTrans}) and (\ref{eq:targetDisNormal}), respectively, with mean $\mathbf{\mu}(s_t,a_t) = ({\theta_1}_t+ \dot{\theta_1}_t\tau,\dot{\theta_1}_t+\ddot{\theta_1}_t\tau,{\theta_2}_t+ \dot{\theta_2}_t\tau,\dot{\theta_2}_t+\ddot{\theta_2}_t\tau,)$, where  $\tau$ is the discretizated time interval and $\ddot{\theta_1}_t, \ddot{\theta_2}_t$ can be obtained from (\ref{eq:acrobot}) by substituting $\theta_t,\dot{\theta}_t,a_t,$ and co-variance $\Sigma(s_t,a_t) = \Sigma$.

As the objective is to stabilize the acrobot to the upright position, we define the reward function as
\begin{align*}
    r(\theta_1,\dot{\theta}_1, \theta_2,\dot{\theta}_2,a)= \exp(-\cos\theta_1-1)+\exp(-\cos(\theta_1+\theta_2)-1).
\end{align*}

\subsection{Comparison with Deep RL Algorithms}
We provided results combining HMC sampling with benchmark Deep RL algorithms DQN and DDPG. We used the same network architecture of DQN and DDPG presented in the  original papers \cite{mnih2015human,lillicrap2015continuous}. To train the networks, we used the Adam optimizer \cite{kingma2014adam} with learning rate $1e^{-5}$, discount coefficient $\gamma=0.99$, and batchsize 32. For all results provided in this paper we used following hyper parameters. Also, we set the number of steps between target network update to 10,000.